\documentclass{article}

\newif\ifcompiletikz
\compiletikzfalse

\usepackage{fancyhdr}
\usepackage{color}
\usepackage{algorithm}
\usepackage{algorithmic}
\usepackage{natbib}

\usepackage{times}
\usepackage{graphicx}
\usepackage{hyperref}
\usepackage{subfigure}
\usepackage{xcolor}
\usepackage{url}
\usepackage{fixltx2e}

\usepackage{amssymb, amsmath, amsthm, mathabx, complexity,bm,breqn}
\catcode`\_=12 
\usepackage{epstopdf,cprotect}
\usepackage{tikz,pgfplots}
\graphicspath{{figs/}}
\DeclareMathOperator*{\argmax}{\arg\!\max}

\def\R{\mathbb{R}}
\def\X{\mathcal{X}}
\def\O{\mathcal{O}}
\def\N{\mathcal{N}}
\def\abs#1{\left\lvert#1\right\rvert}
\def\norm#1{\left\lVert#1\right\rVert}
\def\Prgt{\textstyle{\Pr_>}}
\def\Prlt{\textstyle{\Pr_\leq}}
\def\wh#1{\widehat{#1}}
\def\mat#1{\mathbf{#1}}
\def\comment#1{}
\newtheorem{lemma}{Lemma}
\newtheorem{theorem}{Theorem}
\newtheorem{corollary}{Corollary}
\theoremstyle{remark}

\usepackage{eqparbox}

\usepackage{authblk}

\newcommand\blfootnote[1]{%
  \begingroup
  \renewcommand\thefootnote{}\footnote{#1}%
  \addtocounter{footnote}{-1}%
  \endgroup
}

\ifcompiletikz
  \usepgfplotslibrary{external} 
\fi

\title{Gaussian Process Optimization with Mutual Information}
\author[1]{Emile Contal}
\author[2]{Vianney Perchet}
\author[1]{Nicolas Vayatis}
\affil[1]{CMLA, UMR CNRS 8536, ENS Cachan, France}
\affil[2]{LPMA, Universit\'e Paris Diderot, France}
\date{}

\begin{document} 

\maketitle
\blfootnote{Preprint for the 31st International Conference on Machine Learning (ICML 2014)}

\begin{abstract}
In this paper, we analyze a generic algorithm scheme for sequential global optimization
using Gaussian processes.
The upper bounds we derive on the cumulative regret for this generic algorithm improve
by an exponential factor the previously known bounds for algorithms like \textsf{GP-UCB}.
We also introduce the novel
Gaussian Process Mutual Information algorithm (\textsf{GP-MI}),
which significantly improves further these upper bounds for the cumulative regret.
We confirm the efficiency of this algorithm on synthetic and real tasks
against the natural competitor, \textsf{GP-UCB}, and also the Expected Improvement heuristic.
\end{abstract} 

\vfill
\pagebreak
\section*{Erratum}
After the publication of our article,
we found an error in the proof of Lemma~\ref{lem:mt} which invalidates the main theorem.
It appears that the information given to the algorithm is not sufficient for the main theorem to hold true.
The theoretical guarantees would remain valid in a setting where the algorithm observes the instantaneous regret
instead of noisy samples of the unknown function.
We describe in this page the mistake and its consequences.

Let $f : \mathcal{X} \to \mathbb{R}$ be the unknown function to be optimized, which is a sample from a Gaussian process.
Let's fix $x^\star, x_1, \dots, x_T \in \mathcal{X}$
and the observations $y_t=f(x_t)+\epsilon_t$ where the noise variables $\epsilon_t$ are independent Gaussian noise $\mathcal{N}(0,\sigma^2)$.
We define the instantaneous regret $r_t=f(x^\star)-f(x_t)$ and,
\[M_T = \sum_{t=1}^T \Big( r_t-\mathbb{E}\big[r_t \mid y_1,\dots,y_{t-1}\big] \Big)\,.\]

In Lemma~\ref{lem:mt}, we claimed that $M_T$ is a Gaussian martingale with respect to $\mathbf{Y}_T=y_1, \dots, y_T$.
Even if $M_t-M_{t-1}$ is a centered Gaussian conditioned on $\mathbf{Y}_{T-1}$,
it is wrong to say that $M_T$ is a martingale since it is not measurable with respect to $\mathbf{Y}_T$.

In order to fix Lemma~\ref{lem:mt}, it is possible to modify $M_T$ and use its natural filtration
$\mathcal{F}_T = \big\{ r_t \big\}_{t \leq T}$ instead of $\mathbf{Y}_T$,
\[M_T = \sum_{t=1}^T \Big( r_t - \mathbb{E}\big[r_t \mid \mathcal{F}_{t-1}\big] \Big)\,.\]
Then $M_T$ is a Gaussian martingale with respect to $\mathcal{F}_T$.
Now to adapt the algorithm for this new quantity
it needs to observe $r_t$ instead of $y_t$ to be able to compute both
the posterior expectation and variance for all $x$ in $\mathcal{X}$:
\[\mu_t(x) = \mathbb{E}\big[ f(x) \mid \mathcal{F}_{t-1}\big] \text{ and }
\sigma_t^2(x) = \mathrm{Var<}\big[ f(x) \mid \mathcal{F}_{t-1}\big]\,.\]

We remark that the experiments performed in this article
are remarkably good in spite of Lemma~\ref{lem:mt} being unproved.
After having discovered the mistake
we were able to build scenarios were the \textsf{GP-MI} algorithm is overconfident and misses the optimum of $f$,
and therefore incurs a linear cumulative regret.
\pagebreak

\section{Introduction}
\label{intro}
Stochastic optimization problems are encountered in numerous real world domains including
engineering design \cite{wang2007},
finance \cite{ziemba2006},
natural sciences \cite{floudas2000},
or in machine learning for selecting models by tuning the parameters of learning algorithms \cite{snoek2012}.
We aim at finding the input of a given system which optimizes the output (or reward).
In this view, an iterative procedure uses the previously acquired measures
to choose the next query predicted to be the most useful.
The goal is to maximize the sum of the rewards received at each iteration,
that is to minimize the cumulative regret
by balancing exploration (gathering information by favoring locations with high uncertainty)
and exploitation (focusing on the optimum by favoring locations with high predicted reward).
This optimization task becomes challenging
when the dimension of the search space is high
and the evaluations are noisy and expensive.
Efficient algorithms have been studied to tackle this challenge
such as multiarmed bandit
\cite{auer2002, kleinberg2004, bubeck2011, audibert2011},
active learning \cite{carpentier2011, chen2013}
or Bayesian optimization \cite{mockus1989, grunewalder2010, srinivas2012, freitas2012}.
The theoretical analysis of such optimization procedures
requires some prior assumptions on the underlying function $f$.
Modeling $f$ as a function distributed from a Gaussian process (GP)
enforces near-by locations to have close associated values,
and allows to control the general smoothness of $f$ with high probability
according to the kernel of the GP \cite{rasmussen2006}.
Our main contribution is twofold:
we propose a generic algorithm scheme for Gaussian process optimization
and we prove sharp upper bounds on its cumulative regret.
The theoretical analysis has a direct impact
on strategies built with the \textsf{GP-UCB} algorithm \cite{srinivas2012}
such as \cite{krause2011, desautels2012, contal2013}.
We suggest an alternative policy which achieves an exponential speed up
with respect to the cumulative regret.
We also introduce a novel algorithm,
the Gaussian Process Mutual Information algorithm (\textsf{GP-MI}),
which improves furthermore upper bounds for the cumulative regret
from $\O(\sqrt{T (\log T)^{d+1}})$ for \textsf{GP-UCB}, the current state of the art,
to the spectacular $\O(\sqrt{(\log T)^{d+1}})$,
where $T$ is the number of iterations, $d$ is the dimension of the input space
and the kernel function is Gaussian.
The remainder of this article is organized as follows.
We first introduce the setup and notations. We define the \textsf{GP-MI} algorithm
in Section \ref{sec:setup_algo}.
Main results on the cumulative regret bounds are presented in Section \ref{sec:main_result}.
We then provide technical details in Section \ref{sec:proof}.
We finally confirm the performances of \textsf{GP-MI}
on real and synthetic tasks compared to the state of the art of GP optimization
and some heuristics used in practice.

\section{Gaussian process optimization and the \textsf{GP-MI} algorithm}
\label{sec:setup_algo}
\subsection{Sequential optimization and cumulative regret}
Let $f:\X \to \R$, where $\X \subset \R^d$ is a compact and convex set,
be the unknown function modeling the system we want to be optimized.
We consider the problem of finding the maximum of $f$ denoted by:
\[f(x^\star) = \max_{x\in\X} f(x)\,,\]
via successive queries $x_1, x_2, \dotso \in \X$.
At iteration $T+1$, the choice of the next query $x_{T+1}$ depends on the previous noisy observations,
$\mat{Y}_T = \{y_1, \dotsc, y_T\}$ at locations $X_T = \{x_1, \dotsc, x_T\}$
where $y_t = f(x_t)+\epsilon_t$ for all $t\leq T$,
and the noise variables $\epsilon_1, \dotsc, \epsilon_T$ are independently distributed as
a Gaussian random variable $\N(0,\sigma^2)$ with zero mean and variance $\sigma^2$.
The efficiency of a policy and its ability to address the exploration/exploitation trade-off
is measured via the cumulative regret $R_T$,
defined as the sum of the instantaneous regret $r_t$,
the gaps between the value of the maximum and the values at the sample locations,
\begin{align*}
  r_t &= f(x^\star) - f(x_t) \text{ for } t \leq T\\
 \text{ and } R_T &= \sum_{t=1}^T f(x^\star) - f(x_t)\,.
\end{align*}
Our aim is to obtain upper bounds on the cumulative regret $R_T$
with high probability.

\subsection{The Gaussian process framework}
\paragraph{Prior assumption on $f$.}
In order to control the smoothness of the underlying function,
we assume that $f$ is sampled from a Gaussian process $\mathcal{GP}(m, k)$ with mean function $m:\X\to\R$
and kernel function $k:\X\times\X\to\R^+$.
We formalize in this manner the prior assumption that high local variations of $f$
have low probability.
The prior mean function is considered without loss of generality
to be zero, as the kernel $k$ can completely define the GP \cite{rasmussen2006}.
We consider the normalized and dimensionless framework introduced by \cite{srinivas2010}
where the variance is assumed to be bounded, that is $k(x,x)\leq 1$ for all $x\in\X$.

\paragraph{Bayesian inference.}
\begin{figure}[t]
  \begingroup
  \tikzset{every picture/.style={scale=.75}}
  \begin{center}
    \ifcompiletikz
      \input{figs/gp_new.tex}
    \else
      \includegraphics[width=.5\columnwidth]{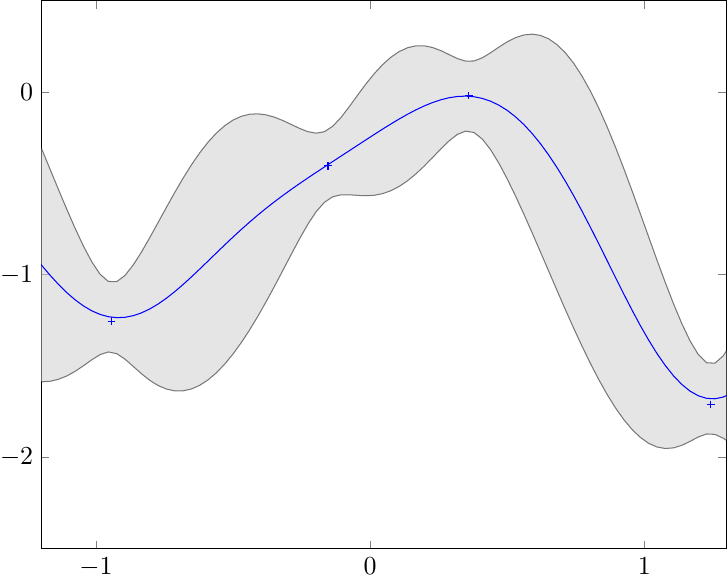}
    \fi
    \caption{One dimensional Gaussian process inference
      of the posterior mean $\mu$ (blue line) and posterior deviation $\sigma$
      (half of the height of the gray envelope)
      with squared exponential kernel,
      based on four observations (blue crosses).}
    \label{fig:gp}
  \end{center}
  \endgroup
\end{figure} 

At iteration $T+1$, given the previously observed noisy values $\mat{Y}_T$ at locations $X_T$,
we use Bayesian inference to compute the current posterior distribution \cite{rasmussen2006},
which is a GP of mean $\mu_{T+1}$ and variance $\sigma_{T+1}^2$
given at any $x \in \X$ by,
\begin{align}
  \label{eq:mu}
  \mu_{T+1}(x) &= \mat{k}_T(x)^\top \mat{C}_T^{-1}\mat{Y}_T\\
  \label{eq:sigma}
  \text{and } \sigma^2_{T+1}(x) &= k(x,x) - \mat{k}_T(x)^\top \mat{C}_T^{-1} \mat{k}_T(x)\,,
\end{align}
where $\mat{k}_T(x) = [k(x_t, x)]_{x_t \in X_T}$ is the vector of covariances
between $x$ and the query points at time $T$,
and $\mat{C}_T = \mat{K}_T + \sigma^2 \mat{I}$
with $\mat{K}_T=[k(x_t,x_{t'})]_{x_t,x_{t'} \in X_T}$ the kernel matrix,
$\sigma^2$ is the variance of the noise
and $\mat{I}$ stands for the identity matrix.
The Bayesian inference is illustrated on Figure \ref{fig:gp} in a sample problem in dimension one,
where the posteriors are based on four observations of a Gaussian Process
with squared exponential kernel.
The height of the gray area represents two posterior standard deviations at each point.

\subsection{The Gaussian Process Mutual Information algorithm}

\paragraph{A novel algorithm.}
The Gaussian Process Mutual Information algorithm (\textsf{GP-MI})
is presented as Algorithm \ref{alg:gpmi}.
The key statement is the choice of the query, $x_t = \argmax_{x\in\X} \mu_t(x) + \phi_t(x)$.
The exploitation ability of the procedure is driven by $\mu_t$,
while the exploration is governed by $\phi_t : \X \to \R$,
which is an increasing function of $\sigma_t^2(x)$.
The novelty in the \textsf{GP-MI} algorithm is that $\phi_t$ is empirically controlled
by the amount of exploration that has been already done,
that is, the more the algorithm has gathered information on $f$,
the more it will focus on the optimum.
In the \textsf{GP-UCB} algorithm from \cite{srinivas2012}
the exploration coefficient is a $\O(\log t)$ and therefore tends to infinity.
The parameter $\alpha$ in Algorithm \ref{alg:gpmi} governs the trade-off between precision and confidence,
as shown in Theorem~\ref{thm:regret}.
The efficiency of the algorithm is robust to the choice of its value.
We confirm empirically this property and provide further discussion on the calibration of $\alpha$
in Section \ref{sec:expes}.

\begin{algorithm}[t]
  \caption{\textsf{GP-MI}}
  \label{alg:gpmi}
  \begin{algorithmic}
   \STATE $\wh{\gamma}_0 \gets 0$
   \FOR{$t=1,2,\dotso$}
     \STATE Compute $\mu_t$ and $\sigma_t^2$ \COMMENT{Bayesian inference (Eq.\,\ref{eq:mu}, \ref{eq:sigma})}
     \STATE $\displaystyle \phi_t(x) \gets \sqrt{\alpha}\left( \sqrt{\sigma_t^2(x) + \wh{\gamma}_{t-1}} - \sqrt{\wh{\gamma}_{t-1}} \right)$ \COMMENT{Definition of $\phi_t(x)$ for all $x\in\X$}
     \STATE $x_t \gets \argmax_{x \in \X} \mu_t(x) + \phi_t(x)$ \COMMENT{Selection of the next query location}
     \STATE $\wh{\gamma}_t \gets \wh{\gamma}_{t-1} + \sigma_t^2(x_t)$ \COMMENT{Update $\wh{\gamma}_t$}
     \STATE Sample $x_t$ and observe $y_t$ \COMMENT{Query}
   \ENDFOR
 \end{algorithmic}
\end{algorithm}

\paragraph{Mutual information.}
The quantity $\wh{\gamma}_T$ controlling the exploration in our algorithm
forms a lower bound on the information acquired on $f$ by the query points $X_T$.
The information on $f$ is formally defined by $I_T(X_T)$,
the mutual information between $f$ and the noisy observations $\mat{Y}_T$ at $X_T$,
hence the name of the \textsf{GP-MI} algorithm.
For a Gaussian process distribution
$I_T(X_T)=\frac{1}{2} \log \det(\mat{I} + \sigma^{-2} \mat{K}_T)$
where $\mat{K}_T$ is the kernel matrix $[k(x_i,x_j)]_{x_i,x_j\in X_T}$.
We refer to \cite{cover1991} for further reading on mutual information.
We denote by $\gamma_T = \max_{X_T\subset \X : \abs{X_T}=T} I_T(X_T)$
the maximum mutual information obtainable by a sequence of $T$ query points.
In the case of Gaussian processes with bounded variance, the following inequality is satisfied
(Lemma 5.4 in \cite{srinivas2012}):
\begin{equation}
\label{eq:gamma_hat}
  \wh{\gamma}_T = \sum_{t=1}^T \sigma_t^2(x_t) \leq C_1 \gamma_T
\end{equation}
where $C_1 = \frac{2}{\log(1+\sigma^{-2})}$
and $\sigma^2$ is the noise variance.
The upper bounds on the cumulative regret we derive in the next section
depend mainly on this key quantity.

\begin{algorithm}[t]
  \caption{Generic Optimization Scheme $(\phi_t)$}
  \label{alg:gpopt}
  \begin{algorithmic}
   \FOR{$t=1,2,\dotso$}
     \STATE Compute $\mu_t$ and $\phi_t$
     \STATE $x_t \gets \argmax_{x \in \X} \mu_t(x) + \phi_t(x)$
     \STATE Sample $x_t$ and observe $y_t$
   \ENDFOR
 \end{algorithmic}
\end{algorithm}

\section{Main Results}
\label{sec:main_result}
\subsection{Generic Optimization Scheme}
We first consider the generic optimization scheme defined in Algorithm \ref{alg:gpopt},
where we let $\phi_t$ as a generic function viewed as a parameter of the algorithm.
We only require $\phi_t$ to be measurable with respect to $\mat{Y}_{t-1}$,
the observations available at iteration $t-1$.
The theoretical analysis of Algorithm \ref{alg:gpopt} can be used as a
plug-in theorem for existing algorithms.
For example the \textsf{GP-UCB} algorithm with parameter $\beta_t = \O(\log t)$
is obtained with $\phi_t(x) = \sqrt{\beta_t \sigma_t^2(x)}$.
A generic analysis of Algorithm \ref{alg:gpopt} leads to the following
upper bounds on the cumulative regret with high probability.

\begin{theorem}[Regret bounds for the generic algorithm]
\label{thm:general_regret}
For all $\delta>0$ and $T>0$,
the regret $R_T$ incurred by Algorithm \ref{alg:gpopt} 
on $f$ distributed as a GP perturbed by independent Gaussian noise with variance $\sigma^2$
satisfies the following bound with high probability,
with $C_1 = \frac{2}{\log(1+\sigma^{-2})}$ and $\alpha = \log\frac{2}{\delta}$:
\begin{equation*}
\Pr\left[ R_T \leq \sum_{t=1}^T \big(\phi_t(x_t) -\phi_t(x^\star) \big)
     + 4\sqrt{\alpha (C_1\gamma_T+1)} + \frac{\sqrt{\alpha}}{2}\frac{\sum_{t=1}^T \sigma_t^2(x^\star)}{\sqrt{C_1\gamma_T+1}} \right]
   \geq 1-\delta \,.
 \end{equation*}
\vspace{0.2em}
\end{theorem}

The proof of Theorem~\ref{thm:general_regret},
relies on concentration guarantees for Gaussian processes
(see Section \ref{sec:proof_general}).
Theorem~\ref{thm:general_regret} provides an intermediate result
used for the calibration of $\phi_t$ to face the exploration/exploitation trade-off.
For example by choosing $\phi_t(x) = \frac{\sqrt{\alpha}}{2} \sigma_t^2(x)$
(where the dimensional constant is hidden),
Algorithm \ref{alg:gpopt} becomes a variant of the \textsf{GP-UCB} algorithm
where in particular the exploration parameter $\sqrt{\beta_t}$ is fixed to $\frac{\sqrt{\alpha}}{2}$
instead of being an increasing function of $t$.
The upper bounds on the cumulative regret with this definition of $\phi_t$
are of the form $R_T = \O(\gamma_T)$,
as stated in Corollary \ref{cor:gpucb}.
We then also consider the case 
where the kernel $k$ of the Gaussian process is under the form of a squared exponential (RBF) kernel,
$k(x_1,x_2) = \exp\Big(-\frac{\norm{x_1-x_2}^2}{2 l^2}\Big)$,
for all $x_1,x_2 \in \X$ and length scale $l \in \R$.
In this setting, the maximum mutual information $\gamma_T$
satisfies the upper bound $\gamma_T = \O\big( (\log T)^{d+1} \big)$,
where $d$ is the dimension of the input space \cite{srinivas2012}.

\begin{corollary}
  \label{cor:gpucb}
  Consider the Algorithm \ref{alg:gpopt} where we set $\phi_t(x) = \frac{\sqrt{\alpha}}{2} \sigma_t^2(x)$.
  Under the assumptions of Theorem~\ref{thm:general_regret},
  we have that the cumulative regret for Algorithm \ref{alg:gpopt}
  satisfies the following upper bounds with high probability:
  \begin{itemize}
  \item
    For $f$ sampled from a GP with general kernel:\\
    $R_T = \O(\gamma_T)$.
  \item
    For $f$ sampled from a GP with RBF kernel:\\
    $R_T = \O\big( (\log T)^{d+1} \big)$.
  \end{itemize}
  \vspace{0.1em}
\end{corollary}
To prove Corollary~\ref{cor:gpucb} we apply Theorem~\ref{thm:general_regret} with the given definition of $\phi_t$
and then Equation~\ref{eq:gamma_hat},
which leads to $\Pr\big[R_T \leq \frac{\sqrt{\alpha}}{2} C_1 \gamma_T + 4\sqrt{\alpha (C_1 \gamma_T + 1)}\big] \geq 1-\delta$.
The previously known upper bounds on the cumulative regret for the \textsf{GP-UCB} algorithm
are of the form $R_T = \O\big(\sqrt{T \beta_T \gamma_T}\big)$
where $\beta_T = \O\big(\log\frac{T}{\delta}\big)$.
The improvement of the generic Algorithm \ref{alg:gpopt} with $\phi_t(x) = \frac{\sqrt{\alpha}}{2} \sigma_t^2(x)$
over the \textsf{GP-UCB} algorithm
with respect to the cumulative regret is then exponential
in the case of Gaussian processes with RBF kernel.
For $f$ sampled from a GP with linear kernel,
corresponding to $f(x) = w^T x$ with $w \sim \N(0,\mat{I})$,
we obtain $R_T = \O\big(d \log T\big)$.
We remark that the GP assumption with linear kernel
is more restrictive than the linear bandit framework,
as it implies a Gaussian prior over the linear coefficients $w$.
Hence there is no contradiction with the lower bounds stated for linear bandit
like those of \cite{dani2008}.
We refer to \cite{srinivas2012} for the analysis of $\gamma_T$ with other kernels widely used in practice.

\subsection{Regret bounds for the \textsf{GP-MI} algorithm}
We present here the main result of the paper,
the upper bound on the cumulative regret for the \textsf{GP-MI} algorithm.

\begin{theorem}[Regret bounds for the \textsf{GP-MI} algorithm]
\label{thm:regret}
For all $\delta>0$ and $T>1$,
the regret $R_T$ incurred by Algorithm \ref{alg:gpmi}
on $f$ distributed as a GP perturbed by independent Gaussian noise with variance $\sigma^2$
satisfies the following bound with high probability,
with $C_1\!=\!\frac{2}{\log(1+\sigma^{-2})}$ and $\alpha\!=\!\log\frac{2}{\delta}$:
\[\Pr\Big[R_T \leq 5\sqrt{\alpha C_1 \gamma_T} + 4\sqrt{\alpha} \Big] \geq 1-\delta\,.\]
\end{theorem}

The proof for Theorem~\ref{thm:regret} is provided in Section \ref{sec:proof_gpmi},
where we analyze the properties of the exploration functions $\phi_t$.
Corollary \ref{cor:rbf_gpmi} describes the case with RBF kernel
for the \textsf{GP-MI} algorithm.

\begin{corollary}[RBF kernels]
\label{cor:rbf_gpmi}
The cumulative regret $R_T$ incurred by Algorithm \ref{alg:gpmi}
on $f$ sampled from a GP with RBF kernel satisfies with high probability,
\[R_T = \O\Big((\log T)^{\frac{d+1}{2}}\Big)\,.\]
\end{corollary}

The \textsf{GP-MI} algorithm significantly improves the upper bounds for the cumulative regret
over the \textsf{GP-UCB} algorithm and the alternative policy of Corollary \ref{cor:gpucb}.

\section{Theoretical analysis}
\label{sec:proof}
In this section, we provide the proofs for Theorem~\ref{thm:general_regret}
and Theorem~\ref{thm:regret}.
The approach presented here to study the cumulative regret
incurred by Gaussian process optimization strategies
is general and can be used further for other algorithms.

\subsection{Analysis of the general algorithm}
\label{sec:proof_general}
The theoretical analysis of Theorem~\ref{thm:general_regret} uses a similar approach
to the Azuma-Hoeffding inequality adapted for Gaussian processes.
Let $r_t=f(x^\star)-f(x_t)$ for all $t\leq T$.
We define $M_T$, which is shown later to be a martingale with respect to $\mat{Y}_{T-1}$,
\begin{equation}
\label{eq:m}
M_T = \sum_{t=1}^T \Big(r_t - \big(\mu_t(x^\star) - \mu_t(x_t)\big) \Big)\,,
\end{equation}
for $T \geq 1$ and $M_0 = 0$.
Let $Y_t$ be defined as the martingale difference sequence with respect to $M_T$,
that is the difference between the instantaneous regret and the gap between the posterior mean
for the optimum and the one for the point queried,
\[Y_t = M_t - M_{t-1} = r_t - \big(\mu_t(x^\star) - \mu_t(x_t)\big)
\text{ for } t\geq 1\,. \]

\begin{lemma}
\label{lem:mt}
The sequence $M_T$ is a martingale with respect to $\mat{Y}_{T-1}$ and 
for all $t\leq T$, given $\mat{Y}_{t-1}$, the random variable $Y_t$ is distributed as a Gaussian $\N(0, \ell_t^2)$
with zero mean and variance $\ell_t^2$, where:
\begin{equation}
\label{eq:ell}
 \ell_t^2 = \sigma_t^2(x^\star) + \sigma_t^2(x_t) -2 k(x^\star, x_t)\,.
\end{equation}
\end{lemma}

\begin{proof}
From the GP assumption,
we know that given $\mat{Y}_{t-1}$, 
the distribution of $f(x)$ is Gaussian $\N \big(\mu_t(x), \sigma_t^2(x)\big)$ for all $x\in\X$,
and $r_t$ is a projection of a Gaussian random vector,
that is $r_t$ is distributed as a Gaussian $\N\big(\mu_t(x^\star)-\mu_t(x_t), \ell_t^2\big)$
and $Y_t$ is distributed as Gaussian $\N(0, \ell_t^2)$,
with $\ell_t^2 = \sigma_t^2(x^\star) + \sigma_t^2(x_t) -2 k(x^\star, x_t)$,
hence $M_T$ is a Gaussian martingale.
\end{proof}

We now give a concentration result for $M_T$ using inequalities for self-normalized martingales.
\begin{lemma}
\label{lem:martingale}
For all $\delta>0$ and $T>1$, the martingale $M_T$
normalized by the predictable quadratic variation $\sum_{t=1}^T\ell_t^2$
satisfies the following concentration inequality with $\alpha=\log \frac{2}{\delta}$
and $y=8(C_1\gamma_T+1)$:
\begin{equation*}
\Pr\left[ M_T \leq \sqrt{2\alpha y} + \sqrt{\frac{2\alpha}{y}} \sum_{t=1}^T \sigma_t^2(x^\star) \right] \geq 1-\delta\,.
\end{equation*}
\end{lemma}

\begin{proof}
Let $y = 8(C_1 \gamma_T+1)$.
We introduce the notation $\Prgt[A]$ for $\Pr[A \land \sum_{t=1}^T \ell_t^2 > y]$
and $\Prlt[A]$ for $\Pr[A \land \sum_{t=1}^T \ell_t^2 \leq y]$.
Given that $M_t$ is a Gaussian martingale,
using Theorem~4.2 and Remark~4.2 from \cite{bercu2008}
with $\langle M \rangle_T = \sum_{t=1}^T \ell_t^2$
and $a=0$ and $b=1$ we obtain for all $x>0$:
\[ \Prgt\left[ \frac{M_T}{\sum_{t=1}^T \ell_t^2} > x \right]
  < \exp\left( -\frac{x^2 y}{2} \right)\,.\]
With $x = \sqrt{\frac{2\alpha}{y}}$ where $\alpha=\log \frac{2}{\delta}$,
we have:
\[\Prgt\left[M_T > \sqrt{\frac{2\alpha}{y}} \sum_{t=1}^T \ell_t^2\right] < \frac{\delta}{2}\,.\]
By definition of $\ell_t$ in Eq.\,\ref{eq:ell}
and with $k(x^\star, x_t) \geq 0$, we have for all $t\geq 1$ that
$\ell_t^2 \leq \sigma_t^2(x_t) + \sigma_t^2(x^\star)$.
Using Equation~\ref{eq:gamma_hat} we have $\wh{\gamma}_T \leq \frac{y}{8}$,
we finally get:
\begin{equation}
\label{eq:case_gt}
\Prgt\left[ M_T > \frac{\sqrt{2\alpha y}}{8}
  + \sqrt{\frac{2\alpha}{y}}\sum_{t=1}^T \sigma_t^2(x^\star) \right] < \frac{\delta}{2}\,.
\end{equation}
Now, using Theorem~4.1 and Remark~4.2 from \cite{bercu2008}
the following inequality is satisfied for all $x>0$:
\[ \Prlt\left[ M_T > x \right] < \exp\left( -\frac{x^2}{2y} \right)\,.\]
With $x = \sqrt{2\alpha y}$ we have:
\begin{equation}
\label{eq:case_lt}
\Prlt\left[M_T > \sqrt{2 \alpha y} \right]< \frac{\delta}{2}\,.
\end{equation}

Combining Equations~\ref{eq:case_gt} and \ref{eq:case_lt} leads to,
\[\Pr\left[M_T > \sqrt{2 \alpha y} + \sqrt{\frac{2\alpha}{y}}\sum_{t=1}^T \sigma_t^2(x^\star)\right] < \delta~,\]
proving Lemma~\ref{lem:martingale}.
\end{proof}

The following lemma concludes the proof of Theorem~\ref{thm:general_regret}
using the previous concentration result
and the properties of the generic Algorithm \ref{alg:gpopt}.
\begin{lemma}
\label{lem:regret}
The cumulative regret for Algorithm \ref{alg:gpopt} on $f$ sampled from a GP
satisfies the following bound for all $\delta>0$ and $\alpha$ and $y$ defined in Lemma~\ref{lem:martingale}:
\begin{equation*}
  \Pr\left[ R_T \leq \sum_{t=1}^T \big( \phi_t(x_t) - \phi_t(x^\star) \big) +
    \sqrt{2\alpha y} + \sqrt{\frac{2\alpha}{y}}\sum_{t=1}^T \sigma_t^2(x^\star)
  \right] \geq 1-\delta \,.
\end{equation*}
\end{lemma}

\begin{proof}
By construction of the generic Algorithm \ref{alg:gpopt},
we have $x_t = \argmax_{x \in \X} \mu_t(x) + \phi_t(x)$,
which guarantees for all $t\geq 1$ that $\mu_t(x^\star) - \mu_t(x_t) \leq \phi_t(x_t) - \phi_t(x^\star)$.
Replacing $M_T$ by its definition in Eq.\,\ref{eq:m}
and using the previous property in Lemma~\ref{lem:martingale}
proves Lemma~\ref{lem:regret}.
\end{proof}

\subsection{Analysis of the \textsf{GP-MI} algorithm}
\label{sec:proof_gpmi}
In order to bound the cumulative regret for the \textsf{GP-MI} algorithm,
we focus on an alternative definition of the exploration functions $\phi_t$
where the last term is modified inductively
so as to simplify the sum $\sum_{t=1}^T \phi_t(x_t)$ for all $T>0$.
Being a constant term for a fixed $t>0$, Algorithm \ref{alg:gpmi} remains unchanged.
Let $\phi_t$ be defined as,
\[\phi_t(x) = \sqrt{\alpha ( \sigma_t^2(x) + \wh{\gamma}_{t-1})} - \sum_{i=1}^{t-1} \phi_i(x_i)\,,\]
where $x_t$ is the point selected by Algorithm \ref{alg:gpmi} at iteration $t$.
We have for all $T>1$,
\begin{equation}
  \label{eq:phi_t}
  \sum_{t=1}^T \phi_t(x_t) = \sqrt{\alpha \wh{\gamma}_T} -
       \sum_{t=1}^{T-1}\phi_t(x_t) + \sum_{t=1}^{T-1}\phi_t(x_t) 
  = \sqrt{\alpha \wh{\gamma}_T}\,.
\end{equation}
We can now derive upper bounds for $\sum_{t=1}^T\big(\phi_t(x_t)-\phi_t(x^\star)\big)$
which will be plugged in Theorem~\ref{thm:general_regret}
in order to cancel out the terms involving $x^\star$.
In this manner we can calibrate sharply the exploration/exploitation trade-off
by optimizing the remaining terms.

\begin{lemma}
\label{lem:gp_mi}
For the \textsf{GP-MI} algorithm,
the exploration term in the equation of Theorem~\ref{thm:general_regret}
satisfies the following inequality:
\[
\sum_{t=1}^T\big(\phi_t(x_t)-\phi_t(x^\star)\big) \leq 
  \sqrt{\alpha \wh{\gamma}_T} - \frac{\sqrt{\alpha}}{2}\frac{\sum_{t=1}^T\sigma_t^2(x^\star)}{\sqrt{\wh{\gamma}_T+1}}\,.
\]
\end{lemma}
\begin{proof}
Using our alternative definition of $\phi_t$
which gives the equality stated in Equation~\ref{eq:phi_t}, we know that,
\begin{equation*}
  \sum_{t=1}^T \big( \phi_t(x_t) - \phi_t(x^\star) \big) =
  \sqrt{\alpha} \Bigg(
  \sqrt{\wh{\gamma}_T} +
  \sum_{t=1}^T \Big(
    \sqrt{\wh{\gamma}_{t-1}}-\sqrt{\wh{\gamma}_{t-1}+\sigma_t^2(x^\star)} \Big) \Bigg)\,.
\end{equation*}

By concavity of the square root,
we have for all $a\geq-b$
that $\sqrt{a + b} - \sqrt{a} \leq \frac{b}{2 \sqrt{a}}$.
Introducing the notations $a_t = \wh{\gamma}_{t-1}+\sigma_t^2(x^\star)$
and $b_t = -\sigma_t^2(x^\star)$,
we obtain,
\[
\sum_{t=1}^T \big( \phi_t(x_t) - \phi_t(x^\star) \big)
  \leq \sqrt{\alpha \wh{\gamma}_T} + \frac{\sqrt{\alpha}}{2} \sum_{t=1}^T \frac{b_t}{\sqrt{a_t}}\,.
\]
Moreover, with $0\leq \sigma_t^2(x) \leq 1$ for all $x\in \X$,
we have $a_t\leq \wh{\gamma}_T+1$ and $b_t \leq 0$ for all $t\leq T$ which gives,
\[
\sum_{t=1}^T \frac{b_t}{\sqrt{a_t}} \leq - \frac{\sum_{t=1}^T \sigma_t^2(x^\star)}{\sqrt{\wh{\gamma}_T+1}}\,,
\]
leading to the inequality of Lemma~\ref{lem:gp_mi}.
\end{proof}

The following lemma combines the results from Theorem~\ref{thm:general_regret} and Lemma~\ref{lem:gp_mi}
to derive upper bounds on the cumulative regret for the \textsf{GP-MI} algorithm with high probability.

\begin{lemma}
\label{lem:regret_gpmi}
The cumulative regret for Algorithm \ref{alg:gpmi} on f sampled from a GP
satisfies the following bound for all $\delta>0$ and $\alpha$ defined in Lemma~\ref{lem:martingale},
\[
\Pr\left[ R_T \leq 5\sqrt{\alpha C_1 \gamma_T} + 4\sqrt{\alpha} \right] \geq 1-\delta\,.
\]
\end{lemma}
\begin{proof}
Considering Theorem~\ref{thm:general_regret} in the case of the \textsf{GP-MI} algorithm and
bounding $\sum_{t=1}^T\big(\phi_t(x_t)-\phi_t(x^\star)\big)$ with Lemma~\ref{lem:gp_mi},
we obtain the following bound on the cumulative regret incurred by \textsf{GP-MI}:
\begin{dmath*}
  \Pr\left[ R_T \leq \sqrt{\alpha \wh{\gamma}_T} - \frac{\sqrt{\alpha}}{2} \frac{\sum_{t=1}^T
    \sigma_t^2(x^\star) }{\sqrt{\wh{\gamma}_T+1}} + 4\sqrt{\alpha (C_1 \gamma_T+1)}
    + \frac{\sqrt{\alpha}}{2} \frac{\sum_{t=1}^T\sigma_t^2(x^\star)}{\sqrt{C_1 \gamma_T+1}} \right] \geq 1-\delta \,,
\end{dmath*}
which simplifies to the inequality of Lemma~\ref{lem:regret_gpmi} using Equation~\ref{eq:gamma_hat},
and thus proves Theorem~\ref{thm:regret}.
\end{proof}

\section{Practical considerations and experiments}
\label{sec:expes}
\subsection{Numerical experiments}


\paragraph{Protocol.}
We compare the empirical performances of our algorithm
against the state-of-the-art of GP optimization,
the \textsf{GP-UCB} algorithm \cite{srinivas2012},
and a commonly used heuristic, the Expected Improvement (\textsf{EI}) algorithm with GP \cite{jones1998}.
The tasks used for assessment come from
two real applications and five synthetic problems described here.
For all data sets and algorithms the learners were initialized
with a random subset of $10$ observations $\big\{(x_i, y_i)\big\}_{i\leq 10}$.
When the prior distribution of the underlying function was not known,
the Bayesian inference was made using a squared exponential kernel.
We first picked the half of the data set to estimate the hyper-parameters of the kernel
via cross validation in this subset.
In this way, each algorithm was running with the same prior information.
The value of the parameter $\delta$ for the \textsf{GP-MI} and the \textsf{GP-UCB} algorithms
was fixed to $\delta=10^{-6}$ for all these experimental tasks.
Modifying this value by several orders of magnitude
is insignificant with respect to the empirical mean cumulative regret
incurred by the algorithms,
as discussed in Section \ref{sec:practical_aspects}.
The results are provided in Figure \ref{fig:expes}.
The curves show the evolution of the average regret $\frac{R_T}{T}$
in term of iteration $T$.
We report the mean value with the confidence interval over a hundred experiments.

\paragraph{Description of the data sets.}
\begin{figure}[t]
  \begin{center}
    \begin{subfigure}[Gaussian mixture]{
        \includegraphics[width=.4\columnwidth]{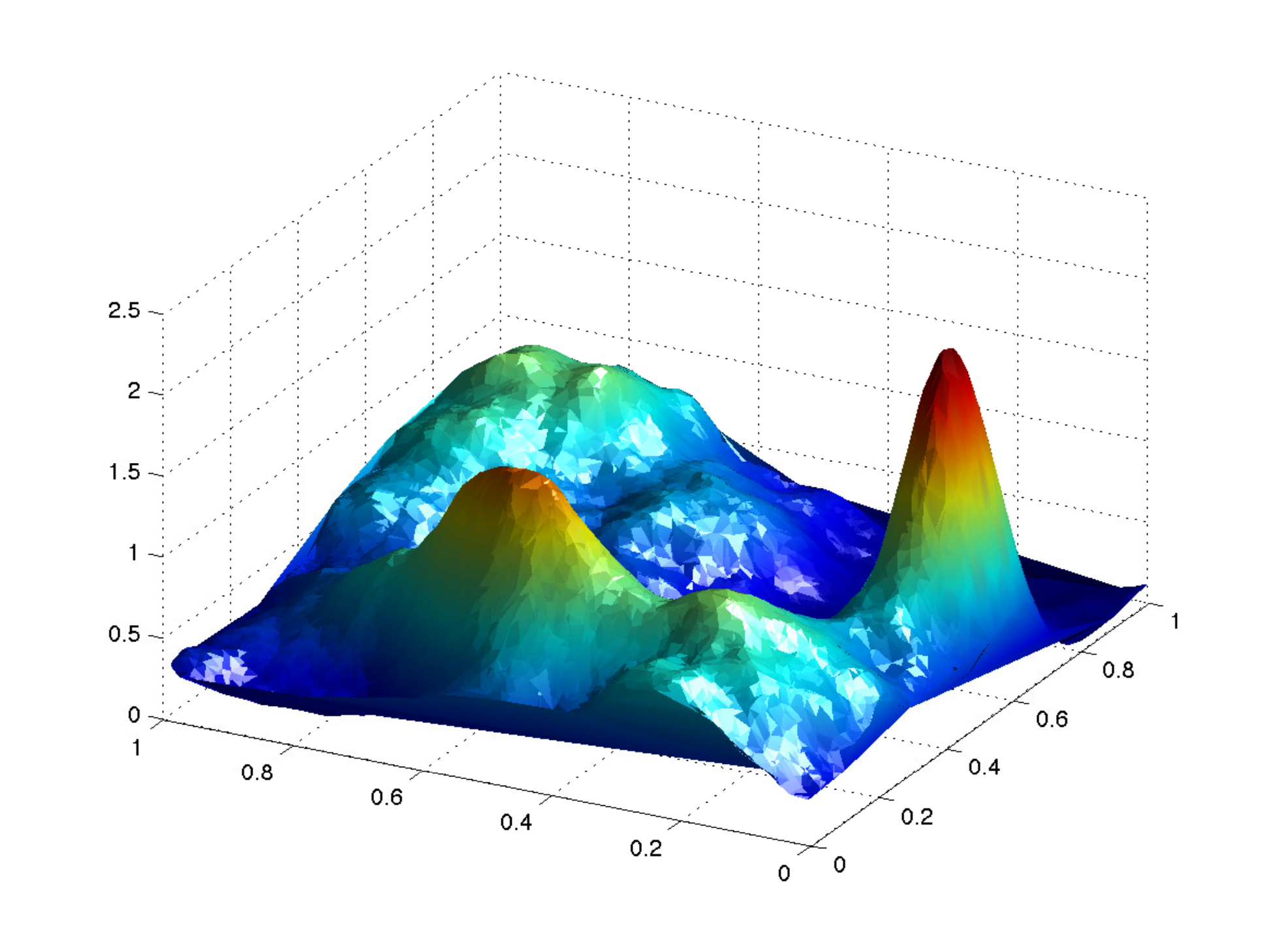}
        \label{fig:gaussian_mix}}
    \end{subfigure}
    \begin{subfigure}[Himmelblau]{
        \includegraphics[width=.4\columnwidth]{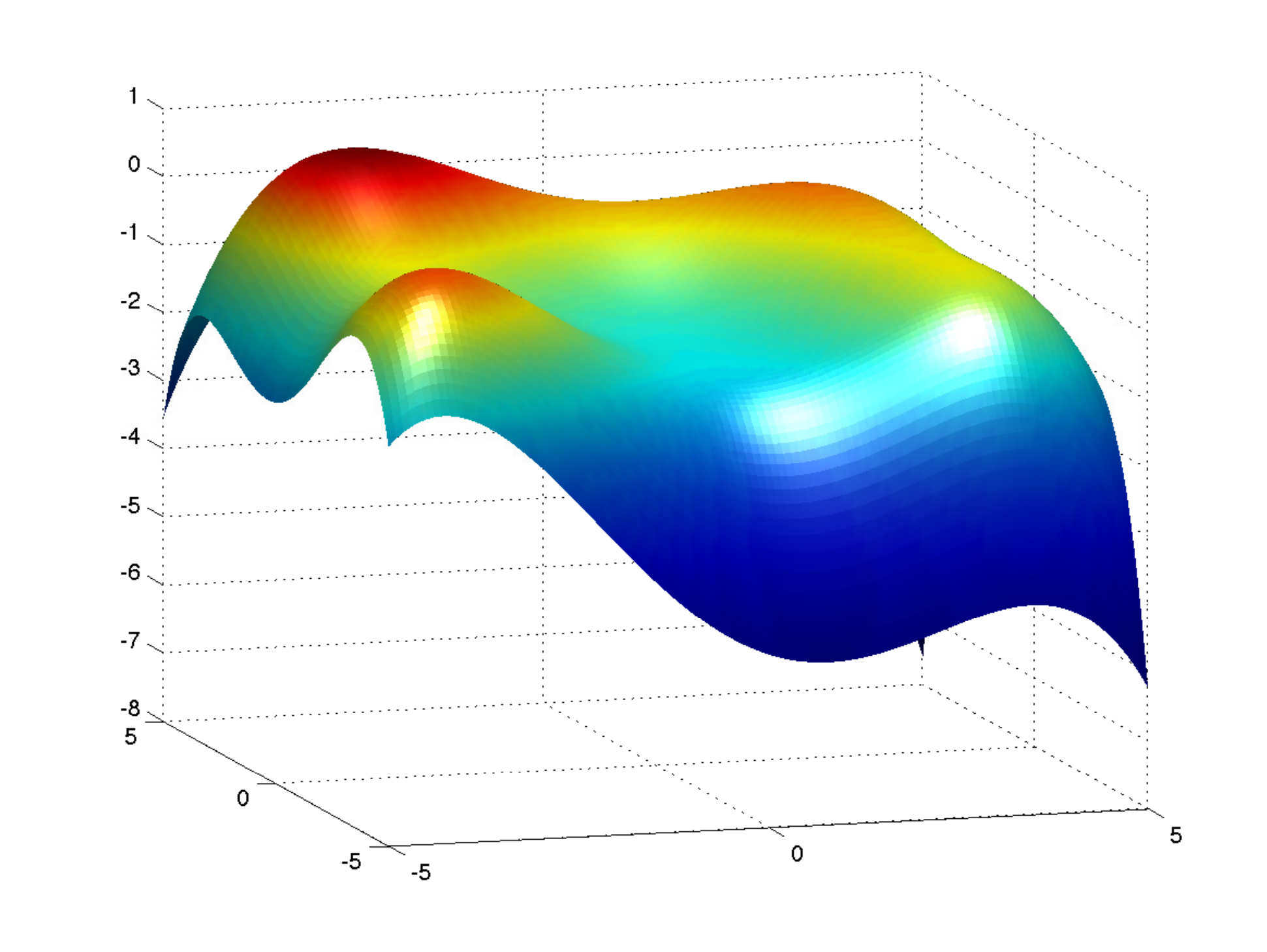}
        \label{fig:himmelblau}}
    \end{subfigure}
    \caption{Visualization of the synthetic functions used for assessment}
  \end{center}
\end{figure} 

We describe briefly all the data sets used for assessment.
\begin{itemize}
\item \emph{Generated GP.}
The generated Gaussian process functions are random GPs drawn from an isotropic Mat\'{e}rn kernel
in dimension $2$ and $4$, with the kernel bandwidth set to $1$ for dimension $2$,
and $16$ for dimension $4$.
The Mat\'{e}rn parameter was set to $\nu=3$ and the noise standard deviation to 1\%
of the signal standard deviation.

\item \emph{Gaussian Mixture.}
This synthetic function comes from the addition of three $2$-D Gaussian functions.
We then perturb these Gaussian functions with smooth variations
generated from a Gaussian Process with isotropic Mat\'{e}rn Kernel and 1\% of noise.
It is shown on Figure \ref{fig:gaussian_mix}.
The highest peak being thin, the sequential search for the maximum of this function
is quite challenging.

\item \emph{Himmelblau.}
This task is another synthetic function in dimension $2$.
We compute a slightly tilted version of the Himmelblau's function with
the addition of a linear function,
and take the opposite to match the challenge of finding its maximum.
This function presents four peaks but only one global maximum.
It gives a practical way to test the ability of a strategy to manage exploration/exploitation trade-offs.
It is represented in Figure \ref{fig:himmelblau}.

\item \emph{Branin.}
The Branin or Branin-Hoo function is a common benchmark function for global optimization.
It presents three global optimum in the $2$-D square $[-5, 10]\times[0, 15]$.
This benchmark is one of the two synthetic functions used by \cite{srinivas2012}
to evaluate the empirical performances of the \textsf{GP-UCB} algorithm.
No noise has been added to the original signal in this experimental task.

\item \emph{Goldstein-Price.}
The Goldstein \& Price function is an other benchmark function for global optimization,
with a single global optimum but several local optima in the $2$-D square $[-2,2]\times[-2,2]$.
This is the second synthetic benchmark used by \cite{srinivas2012}.
Like in the previous challenge, no noise has been added to the original signal.

\item \emph{Tsunamis.}
Recent post-tsunami survey data as well as the numerical simulations of \cite{hill2012}
have shown that in some cases the run-up,
which is the maximum vertical extent of wave climbing on a beach,
in areas which were supposed to be protected by small islands in the vicinity of coast,
was significantly higher than in neighboring locations.
Motivated by these observations \cite{stefanakis2012} investigated this phenomenon
by employing numerical simulations using the VOLNA code \cite{dutykh2011}
with the simplified geometry of a conical island sitting on a flat surface in front of a sloping beach.
In the study of \cite{stefanakis2013} the setup was controlled by five physical parameters
and the aim was to find with confidence and with the least number of simulations
the parameters leading to the maximum run-up amplification.

\item \emph{Mackey-Glass function.}
The Mackey-Glass delay-differential equation
is a chaotic system in dimension $6$, but without noise.
It models real feedback systems and is used in physiological domains such as
hematology, cardiology, neurology, and psychiatry.
The highly chaotic behavior of this function makes it an exceptionally difficult optimization problem.
It has been used as a benchmark for example by \cite{flake2002}.
\end{itemize}

\pagebreak
\paragraph{Empirical comparison of the algorithms.}

Figure \ref{fig:expes} compares the empirical mean average regret $\frac{R_T}{T}$ for the three algorithms.
On the easy optimization assessments like the Branin data set (Fig.\,\ref{fig:expe_branin})
the three strategies behave in a comparable manner,
but the \textsf{GP-UCB} algorithm incurs a larger cumulative regret.
For more difficult assessments the \textsf{GP-UCB} algorithm performs poorly
and our algorithm always surpasses the \textsf{EI} heuristic.
The improvement of the \textsf{GP-MI} algorithm against the two competitors
is the most significant for exceptionally challenging optimization tasks
as illustrated in Figures \ref{fig:expe_gp2} to \ref{fig:expe_himmelblau} and \ref{fig:expe_mg},
where the underlying functions present several local optima.
The ability of our algorithm to deal with the exploration/exploitation trade-off
is emphasized by these experimental results as its average regret decreases directly after the first iterations,
avoiding unwanted exploration like \textsf{GP-UCB} on Figures 
\ref{fig:expe_gp2} to \ref{fig:expe_himmelblau},
or getting stuck in some local optimum like \textsf{EI} on Figures
\ref{fig:expe_gaussian_mix}, \ref{fig:expe_tsunamis} and \ref{fig:expe_mg}.
We further mention that the \textsf{GP-MI} algorithm 
is empirically robust against the number of dimensions of the data set
(Fig.\,\ref{fig:expe_gp4}, \ref{fig:expe_tsunamis}, \ref{fig:expe_mg}).

\begin{figure}[H]
  \begingroup
  \tikzset{every picture/.style={scale=.6}}
  \centering
  \begin{subfigure}[Generated GP\! ($d\!=\!2$)]{
      \ifcompiletikz
        \input{figs/generated_gp_d2.tikz}
      \else
        \includegraphics[height=.27\columnwidth]{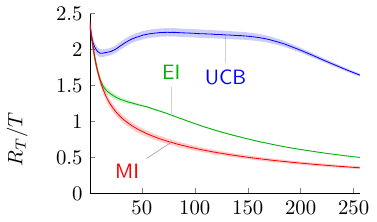}
      \fi
      \label{fig:expe_gp2}}
  \end{subfigure}
  \begin{subfigure}[Generated GP\! ($d\!=\!4$)]{
      \ifcompiletikz
        \input{figs/generated_gp_d4.tikz}
      \else
        \includegraphics[height=.27\columnwidth]{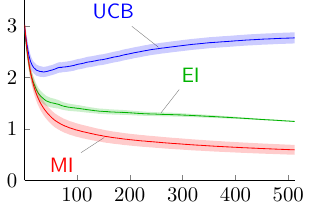}
      \fi
      \label{fig:expe_gp4}}
  \end{subfigure}\\
  \begin{subfigure}[Gaussian Mixture]{
      \ifcompiletikz
        \input{figs/gaussian_mix.tikz}
      \else
        \includegraphics[height=.27\columnwidth]{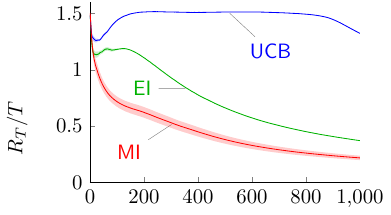}
      \fi
      \label{fig:expe_gaussian_mix}}
  \end{subfigure}
  \begin{subfigure}[Himmelblau]{
      \ifcompiletikz
        \input{figs/himmelblau.tikz}
      \else
        \includegraphics[height=.27\columnwidth]{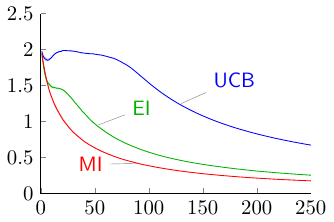}
      \fi
      \label{fig:expe_himmelblau}}
  \end{subfigure}\\
  \begin{subfigure}[Branin]{
      \ifcompiletikz
        \input{figs/branin.tikz}
      \else
        \includegraphics[height=.27\columnwidth]{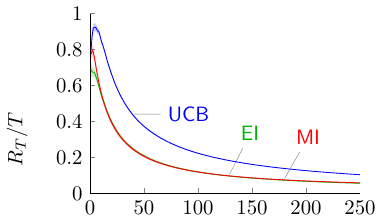}
      \fi
      \label{fig:expe_branin}}
  \end{subfigure}
  \begin{subfigure}[Goldstein]{
      \ifcompiletikz
        \input{figs/goldstein.tikz}
      \else
        \includegraphics[height=.27\columnwidth]{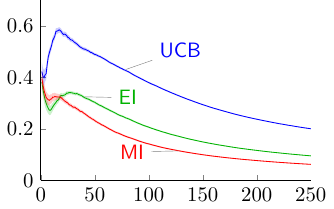}
      \fi
      \label{fig:expe_goldstein}}
  \end{subfigure}\\
  \begin{subfigure}[Tsunamis]{
      \ifcompiletikz
        \input{figs/tsunamis.tikz}
      \else
        \includegraphics[height=.27\columnwidth]{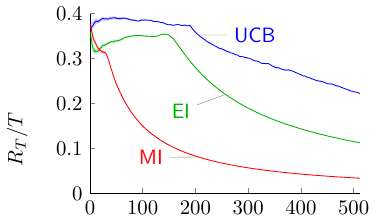}
      \fi
      \label{fig:expe_tsunamis}}
  \end{subfigure}
  \begin{subfigure}[Mackey-Glass]{
      \ifcompiletikz
        \input{figs/mg.tikz}
      \else
        \includegraphics[height=.27\columnwidth]{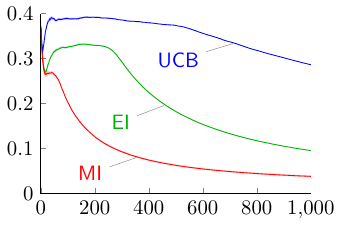}
      \fi
      \label{fig:expe_mg}}
  \end{subfigure}\\
  \endgroup
  \caption{Empirical mean and confidence interval of the average regret $\frac{R_T}{T}$ in term of iteration $T$
    on real and synthetic tasks
    for the \textsf{GP-MI} and \textsf{GP-UCB} algorithms and the \textsf{EI} heuristic
  (lower is better).}
  \label{fig:expes}
\end{figure}

\begin{figure}[t]
  \begingroup
  \tikzset{every picture/.style={scale=.8}}
  \begin{center}
    \ifcompiletikz
      \input{figs/delta.tikz}
    \else
      \includegraphics[width=.5\columnwidth]{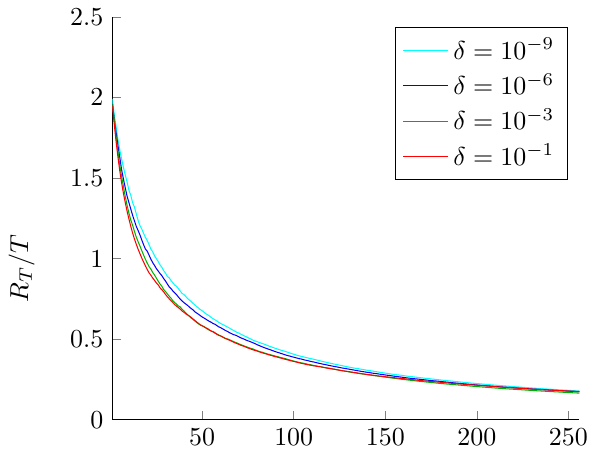}
    \fi
  \end{center}
  \endgroup
  \caption{Small impact of the value of $\delta$ on the mean average regret of the 
    \textsf{GP-MI} algorithm running on the Himmelblau data set.}
  \label{fig:delta}
\end{figure}

\subsection{Practical aspects}
\label{sec:practical_aspects}

\paragraph{Calibration of $\alpha$.}
The value of the parameter $\alpha$ is chosen following Theorem~\ref{thm:regret}
as $\alpha=\log\frac{2}{\delta}$ with $0<\delta<1$ being a confidence parameter.
The guarantees we prove in Section \ref{sec:proof_gpmi} on the cumulative regret for the \textsf{GP-MI} algorithm
holds with probability at least $1-\delta$.
With $\alpha$ increasing linearly for $\delta$ decreasing exponentially toward $0$,
the algorithm is robust to the choice of $\delta$.
We present on Figure \ref{fig:delta} the small impact of $\delta$ on the average regret
for four different values selected on a wide range.

\paragraph{Numerical Complexity.}
Even if the numerical cost of \textsf{GP-MI} is insignificant in practice
compared to the cost of the evaluation of $f$,
the complexity of the sequential Bayesian update \cite{Osborne2010} is $\O(T^2)$
and might be prohibitive for large $T$.
One can reduce drastically the computational time
by means of Lazy Variance Calculation \cite{desautels2012},
built on the fact that $\sigma_T^2(x)$ always decreases for increasing $T$ and for all $x\in\X$.
We further mention that approximated inference algorithms
such as the EP approximation and MCMC sampling \cite{kuss2005}
can be used as an alternative if the computational time is a restrictive factor.

\comment{
\paragraph{Stopping criterion.}
One challenging problem we face in practice
is to decide when to stop the iterative strategy.
We have to design an empirical, yet robust, criterion.
One trivial solution is to fix the number of iterations allowed
by a predefined limit.
This is not a suitable solution for the general case,
as one does not know precisely the amount of exploration needed
to be confident about the maximum of $f$ \cite{stefanakis2013}.
The \textsf{GP-MI} algorithm has the inherent property of decreasing
the importance of exploration along the iterations.
Contrary to \textsf{GP-UCB} which augments the amount of exploration with time,
the exploration region of \textsf{GP-MI} shrinks.
Therefore, we can simply monitor the value of $\phi_t(x_t)$
and stop when it becomes small enough,
for example after reaching a given threshold.
}

\section{Conclusion}
We introduced the \textsf{GP-MI} algorithm for GP optimization
and prove upper bounds on its cumulative regret
which improve exponentially the state-of-the-art in common settings.
The theoretical analysis was presented in a generic framework
in order to expand its impact to other similar algorithms.
The experiments we performed on real and synthetic assessments
confirmed empirically the efficiency of our algorithm
against both the theoretical state-of-the-art of GP optimization,
the \textsf{GP-UCB} algorithm,
and the commonly used \textsf{EI} heuristic.

\subsubsection*{Acknowledgements}
The authors would like to thank David Buffoni and Rapha\"{e}l Bonaque
for fruitful discussions.
The authors also thank the anonymous reviewers of the 31st International Conference on Machine Learning
for their detailed feedback.

\bibliography{../../../biblio/biblio}
\bibliographystyle{plainnat}

\end{document}